\documentclass[10pt,twocolumn,letterpaper]{article}

\usepackage{iccv}
\usepackage{times}
\usepackage{epsfig}
\usepackage{graphicx}
\usepackage{amsmath}
\usepackage{amssymb}

\usepackage{amsthm}

\usepackage[pagebackref=true,breaklinks=true,letterpaper=true,colorlinks,bookmarks=false]{hyperref}

\iccvfinalcopy 

\def\iccvPaperID{****} 

\ificcvfinal\fi

\newcommand{\lt}{\left}
\newcommand{\rt}{\right}
\newcommand{\la}{\lt\langle}
\newcommand{\ra}{\rt\rangle}
\newcommand{\ip}[1]{\la #1 \ra}
\newcommand{\R}{\mathbb{R}}

\newcommand{\diag}{\text{diag}\,}
\newcommand{\trace}{\text{trace}\,}

\newcommand{\T}{^{\mathsf{T}}}

\newtheorem{claim}{Claim}

\begin{document}

\title{Learning Robust Representations for Computer Vision}

\author{Peng Zheng\\
University of Washington\\
Seattle, WA 98195-4322, USA\\
{\tt\small zhengp@uw.edu}
\and
Aleksandr Y. Aravkin\\
University of Washington\\
Seattle, WA 98195-4322, USA\\
{\tt\small saravkin@uw.edu}
\and
Karthikeyan Natesan Ramamurthy\\
IBM T.J. Watson Research Center\\
Yorktown Heights, NY 10598, USA\\
{\tt\small knatesa@us.ibm.com}
\and
Jayaraman Jayaraman Thiagarajan\\
Lawrence Livermore National Laboratory\\
Livermore, CA 94550, USA\\
{\tt\small  jjayaram@llnl.gov}
}

\maketitle

\begin{abstract}
Unsupervised learning techniques 
in computer vision often require learning latent representations, 
such as low-dimensional linear and non-linear subspaces. Noise and outliers in the data can frustrate these approaches by obscuring 
the latent spaces. 

Our main goal is deeper understanding and new development of robust approaches for representation learning.
We provide a new interpretation for existing robust approaches and present two specific contributions: a new robust PCA approach, 
which can separate foreground features from dynamic background, 
and a novel robust spectral clustering method, that can cluster facial images with high accuracy.  
Both contributions show superior performance to standard methods on real-world test sets.

\end{abstract}
\section{Introduction}
Supervised learning, and in particular deep learning~\cite{lecun2015deep, schmidhuber2015deep}, 
have been very successful in computer vision. Applications include autoencoders~\cite{vincent2010stacked} that map between noisy and clean images~\cite{xie2012image}, convolutional networks for image/video analysis \cite{karpathy2014large}, 
and generative adversarial networks that synthesize real world-like images~\cite{goodfellow2014generative}.

In contrast,
unsupervised learning still poses significant challenges. 
Broadly, unsupervised learning seeks to discover hidden structure in the data without using ground truth labels, thereby revealing features of interest.

In this paper, we consider unsupervised representation learning methods which can be used along with centroid-based clustering to summarize the data distribution using a few characteristic samples.

We are interested in spectral clustering \cite{ng2002spectral} and subspace clustering \cite{elhamifar2013sparse}; the proposed ideas can also be generalized to deep embedding-based clustering strategies \cite{xie2016unsupervised}.
 {\it Spectral clustering} methods use neighborhood graphs to learn the underlying representation \cite{ng2002spectral}; this approach is used for image segmentation \cite{shi2000normalized, zelnik2005self} and 3D mesh segmentation \cite{liu2004segmentation}.  
 {\it Subspace clustering} methods model the dataset as a union of low-dimensional linear subspaces and utilize sparse and low-rank methods to obtain the representation; 
this model is used for facial clustering and recognition \cite{elhamifar2013sparse, shakhnarovich2011face}.

Learning effective latent representations hinges on accurately modeling noise and outliers. Further, in practice, the data satisfy the structural assumptions (union of subspaces, low rank, etc.)
only approximately. Adopting robust optimization strategies is a natural way to combat these challenges. For example, consider principal component analysis (PCA), a prototypical representation learning method based on matrix factorization. Given low-rank data contaminated by outliers, 
the classical PCA method will fail to find it. Consequently, the robust PCA (rPCA) method~\cite{candes2011robust}, which decomposes data into low rank and sparse components, is preferred in practice, e.g. background/foreground separation~\cite{candes2011robust, sobral2016lrslibrary}. 
Similarly, when data assumed to be from a union of subspaces is contaminated by outliers, allowing for sparse outliers during optimization 
leads to accurate recovery of the subspaces, e.g. face classification~\cite{wright2009robust}. 

Our goal is to develop effective robust formulations for unsupervised representation learning tasks in computer vision; we are interested in complex situations, when the data is corrupted with a combination of sparse outliers and dense noise.

{\bf Contributions.} We first review the relationship between outlier models and 
statistically robust formulations. In particular, we show that the rPCA formulation is equivalent 
to solving a Huber regression problem for low-rank representation learning.
Using this connection, we develop a new nonconvex penalty, dubbed the Tiber, 
designed to aggressively penalize mid-sized residuals. 
In Section \ref{sec:rPCA}, we show that this penalty is well suited for dynamic background separation, outperforming classic rPCA methods.

Our second contribution is to use the design philosophy behind robust low-rank representation learning to develop a new formulation for robust clustering. We formulate classic spectral analysis as an optimization problem, and then modify this problem to be robust to outliers. The advantages are shown using a synthetic clustering example. We then combine robust spectral clustering with robust subspace clustering
to achieve superior performance on face recognition tasks, surpassing prior work without any data pre-processing; see Section~\ref{sec:LatentClustering}, Table~\ref{tb:faces}.
\section{New Penalties for Learning Robust Representations}
\label{sec:rPCA}
Many tasks in computer vision depend on unsupervised representation learning.  
A well-known example is background/foreground separation, 
 often solved by robust principal component analysis (rPCA). rPCA learns low-rank representations by decomposing a data matrix into a sum of low-rank and sparse components. The low-rank component represents the background and the sparse component represents the foreground~\cite{candes2011robust}. 
 
 In this section, we show that rPCA is equivalent to a  
 robust regression problem, and solving a Huber-robust regression~\cite{huber2011robust} for the background representation 
is completely equivalent to the full rPCA solution. 
We use this equivalence to design a new robust penalty (dubbed Tiber)
based on statistical descriptions of the signals of interest. We illustrate the benefits of using this new non-convex penalty for separating foreground from a dynamic background, using real datasets.  

\subsection{Huber in rPCA}
Background/foreground separation is widely used for detecting moving objects in videos from stationary cameras.
A broad range of techniques have been developed to tackle this task,
ranging from simple thresholding \cite{veit2005maximality} to mixtures of Gaussian models\cite{stauffer1999adaptive,evangelio2012splitting,haines2012background}.
In particular, rPCA has been widely adopted to solve this problem \cite{guyon2012robust,otazo2015low}.

\begin{figure}[t]
\begin{center}
\includegraphics[width=1.0\linewidth]{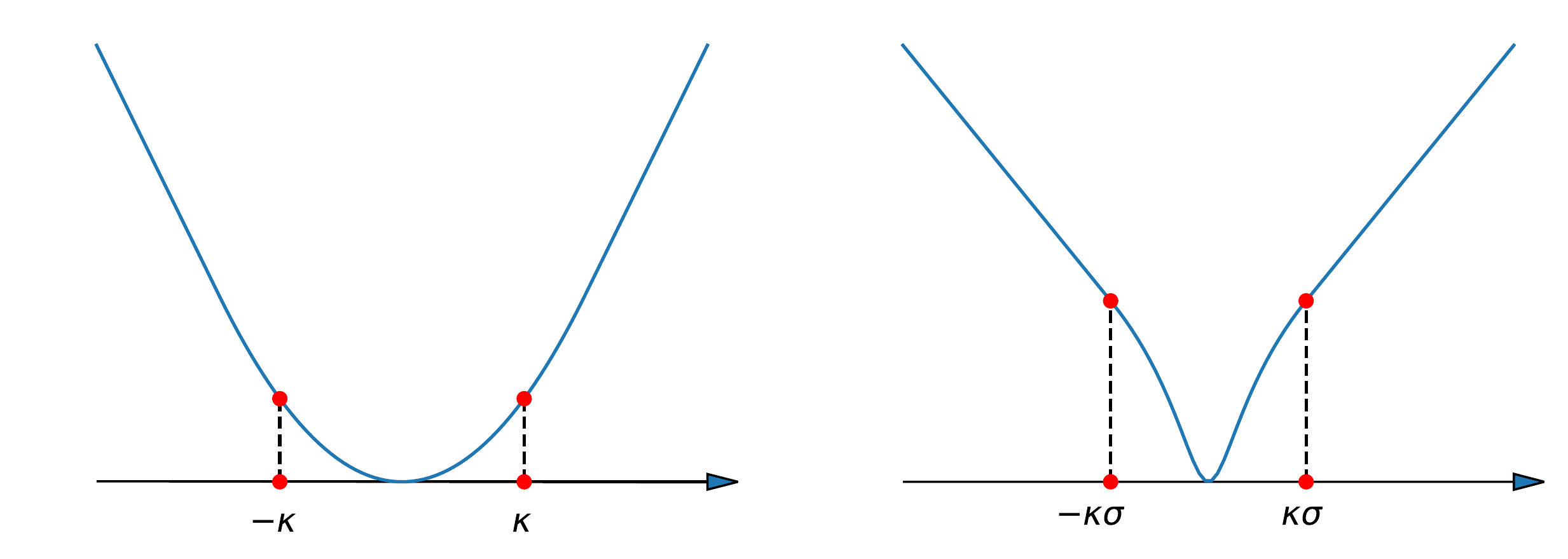}
\end{center}
   \caption{Robust penalties: left: Huber, right: Tiber. Both grow linearly outside an interval containing the origin. 
   The Tiber penalizes `mid-sized' errors within the region far more aggressively than the Huber; such a penalty
   must necessarily be non-convex. }
\label{fig:penalties}
\end{figure}

Denote a given video stream by $Y\in\R^{n \times m}$, 
where each of $m$ frames is reshaped to be a vector of size $n$.
There are many variants of rPCA~\cite{bouwmans2016handbook}. 
We use the {\it stable principal component pursuit} (SPCP) formulation:
\begin{equation}
\label{eq:rPCA}
\min_{L,S}\frac{1}{2}\|L+S-Y\|_F^2 + \kappa\|S\|_1 + \lambda\|L\|_*
\end{equation}
where $L$ represents the background, and  $S$ the foreground. The regularizations used by this formulation 
ensure that $L$ is chosen to be low rank, while $S$ is designed to be sparse; using a quadratic penalty on the 
residual fits of the data up to some error level. 

We can minimize over the variables in any order. Minimizing the first two summands of~\eqref{eq:rPCA} in $S$
gives a closed form function 
\[
\min_S \frac{1}{2}\|L+S-Y\|_F^2 + \kappa\|S\|_1 = \rho(L-Y; \kappa),
\] 
with $\rho(r;\kappa)$ the well-known Huber penalty~\cite{huber2011robust}
\begin{equation}
\label{eq:huber}
\rho(r;\kappa) = 
\begin{cases}
\kappa|r| - \kappa^2/2, & |r| > \kappa\\
r^2/2, & |r| \le \kappa
\end{cases}.
\end{equation}
We provide a simple statement of the following well-known result with a short self-contained proof. 
\begin{claim}
\label{claim:huber}
\begin{equation}
\label{eq:basic_huber}
\rho(r;\kappa) = \min_s\frac{1}{2}(s-r)^2 + \kappa|s|.
\end{equation}
\end{claim}
\begin{proof}
The solution to this optimization problem is the {\it soft threshold} function
(see e.g.~\cite{combettes2011proximal})
\[
\arg\min_s\frac{1}{2}(s-r)^2 + \kappa|s| = \mathbb{S}_\kappa(r) =
\begin{cases}
r - \kappa, & r > \kappa\\
0, & |r| \le \kappa\\
r + \kappa, & r < -\kappa
\end{cases}.
\]
Plugging $ \mathbb{S}_\kappa(r)$ back into~\eqref{eq:basic_huber}, we have 
\[
\frac{1}{2}[\mathbb{S}_\kappa(r) - r]^2 + \kappa|\mathbb{S}_\kappa(r)| = \rho(r;\kappa).
\]
\end{proof}
The optimization problem is separable, so the result immediately extends to the vector case. 
Upon minimization over $S$, problem~\eqref{eq:rPCA} then reduces to 
\begin{equation}
\label{eq:rPCA_huber}
\min_{L} \rho(L-Y;\kappa) + \lambda\|L\|_*.
\end{equation}
To simplify the problem further, we use a factorized representation of $L$~\cite{burer2003nonlinear}, choosing the rank to be $k\ll \min(n,m)$
to obtaining the non-convex formulation
\begin{equation}
\label{eq:rPCA_huber_ncvx}
\min_{U,V} \rho(U\T V - Y;\kappa)
\end{equation}
where $U\in\R^{k\times n}$ and $V\in\R^{k\times m}$. 

\vskip 8pt\noindent
Comparing \eqref{eq:rPCA_huber_ncvx} to \eqref{eq:rPCA} we see two advantages:
\begin{enumerate}
\item The dimension of the decision variable has been reduced from $2nm$ to $k(n+m)$.
\item \eqref{eq:rPCA_huber_ncvx} is smooth, and does not require computing SVDs.
\end{enumerate}
Once we have $U$ and $V$, we can easily recover $L$ and $S$: 
\begin{align*}
L &= U\T V, \quad S = \mathbb{S}_\kappa(U\T V - Y).
\end{align*}
The approach is illustrated in the left panels of Figure~\ref{fig:rpca}.
Although the residual $U^TV-Y$ (shown in row 2) is noisy and not sparse, 
applying $\mathbb{S}_\kappa$ we get the sparse component (row 3), just 
as we would by solving the original formulation~\eqref{eq:rPCA}.

From a statistical perspective, the equivalence of rPCA and Huber means that 
the residual $R = U\T V -Y$, which contains both $S$ and random noise, 
can be modeled by a heavy tailed error distribution.
\begin{claim}
\label{claim:density}
Suppose $\{r_i(x)\}_{i=1}^l$ are i.i.d. samples from a distribution with density
\[p(r;\theta) = \frac{1}{n_c(\theta)}\exp[-\rho(r;\theta)]\]
where $n_c(\theta)= \int_\R \exp[-\rho(r;\theta)]\,dr$ is the normalization constant.
Then maximum likelihood formulation for $x$ is equivalent to the minimization problem
\[\min_{x} \sum_{i=1}^l\rho(r_i(x);\theta).\]
\end{claim}
The claim follows immediately by taking the negative log of the maximum likelihood. 
Claim~\ref{claim:density} means that solving \eqref{eq:rPCA_huber_ncvx} is equivalent to assuming that elements in $R = U\T V - Y$ are
 i.i.d.  samples from the Laplace density
\[
p(r;\kappa) = \frac{1}{n_c(\kappa)}\exp[-\rho(r;\kappa)].
\]
The function $\rho$ has linear tails (See Figure~\ref{fig:penalties}), which means this distribution is much more likely to produce large samples compared to the Gaussian.

\subsection{Weaknesses of the Huber}
Although the Huber distribution can detect sparse outliers, it does not model small errors well.
In many background/foreground separation problems, we must cope with a dynamic background (e.g. motion of tree leaves or water waves).
These small dynamic background perturbations correspond to motion 
we do not care about --- we are much more interested in detecting cars, people, and animals moving through the scene. 

We want to move these dynamics into the low-rank background term. However, the Huber is quadratic 
near the origin (i.e. nearly flat), so small perturbations do not significantly affect the objective value; 
and solving~\eqref{eq:rPCA_huber_ncvx} leaves these terms in the residual $R$. 
Thresholding these terms is either too aggressive (removing features we care about), or too lenient, leaving 
the dynamics in the foreground (see first two columns of Figure~\ref{fig:rpca}).
A better penalty would rise steeply for small values of $R$, without 
significantly affecting tail behavior.

\begin{figure*}[t]
\begin{center}
\includegraphics[width=0.33\linewidth]{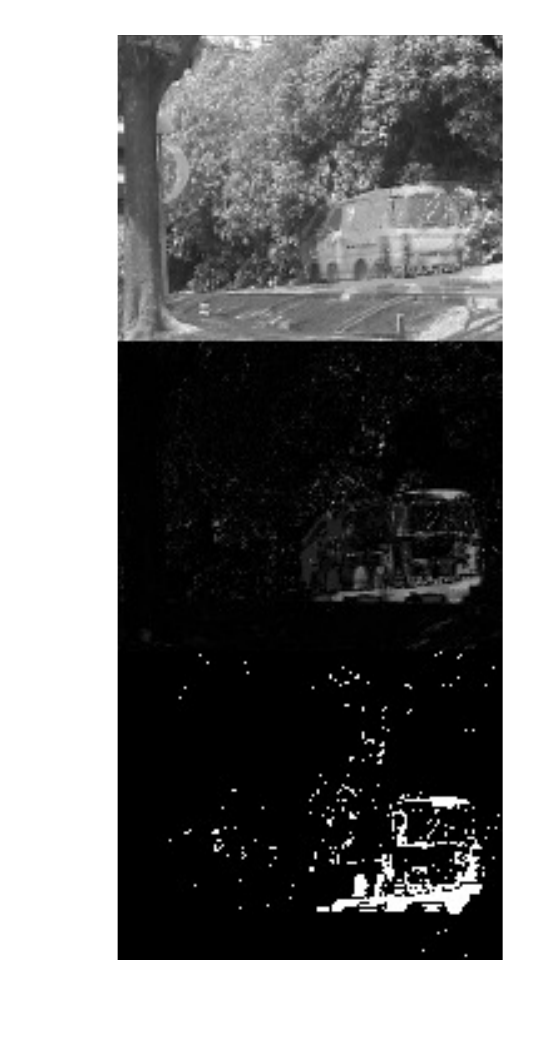}
\hspace{-.5in}
\includegraphics[width=0.33\linewidth]{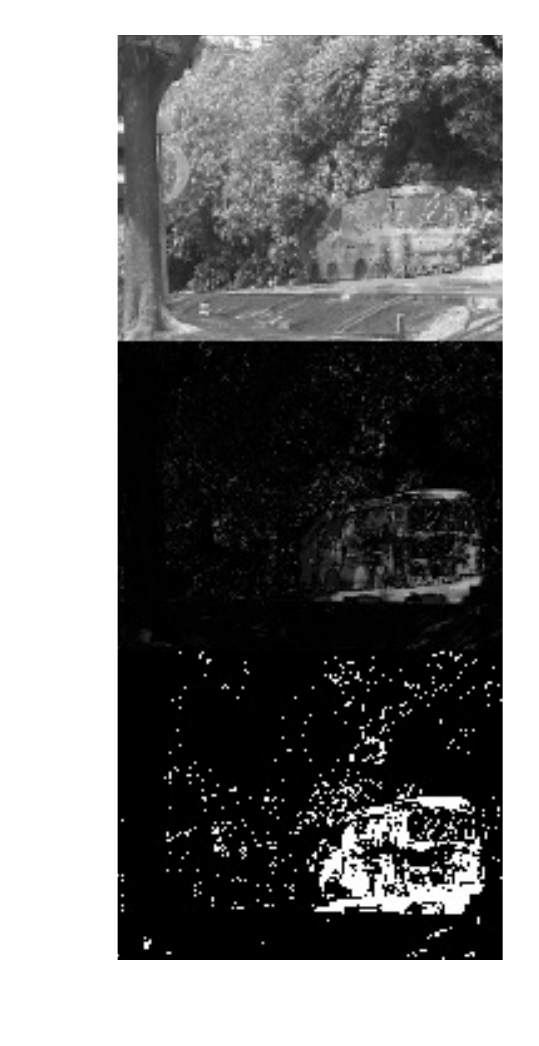}
\hspace{-.5in}
\includegraphics[width=0.33\linewidth]{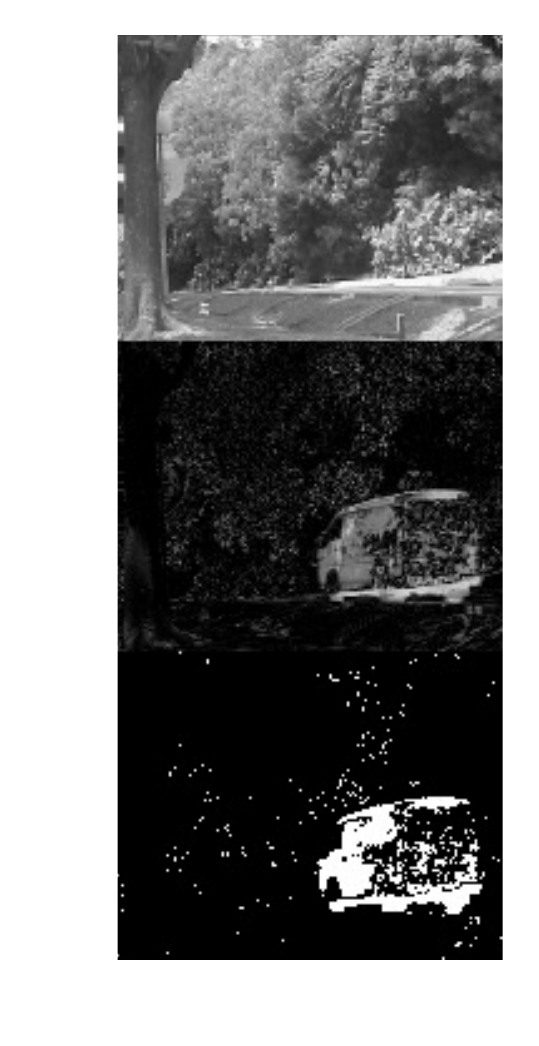}
\end{center}
   \caption{Left: Huber with $\kappa = 0.15$,
   middle: Huber with $\kappa = 0.1$,
   right: Tiber with $\kappa = 10, \sigma=0.03$.
   Row 1: low rank component $L$,
   row 2: residual $|R| = |U\T V-Y|$,
   row 3: binary plot for $S$. The Tiber recovers the van while avoiding the dynamic background.
   }
\label{fig:rpca}
\end{figure*}

\subsection{Tiber for rPCA}
We propose a new penalty, which we call the Tiber. While the Huber is defined by partially minimizing the sum of  
the 1-norm with a quadratic~\eqref{eq:huber}, the Tiber replaces the quadratic with a nonconvex function. 
The resulting penalty can match the tail behavior of Huber, yet have different properties around the origin (see Figure~\ref{fig:penalties}). Tiber is better suited for background/foreground separation problems with dynamic background.
We define the penalty as follows: 
\begin{equation}
\label{eq:tiber}
\begin{aligned}
&\rho_{\mathrm{T}}(r;[\kappa,\sigma]) =\\
&\begin{cases}
\frac{2\kappa}{\sigma(\kappa^2+1)}(|r|-\kappa\sigma) + \log(1+\kappa^2), & |r| > \kappa\sigma\\
\log(1+r^2/\sigma^2), & |r| \le \kappa\sigma
\end{cases}
\end{aligned}
\end{equation}
The Tiber is parametrized by thresholding parameter $\kappa$ and scale parameter $\sigma$.
Just as the Huber, it can be expressed as the value function of a minimization problem. 
We replace the quadratic penalty in Claim~\ref{claim:huber} by the smooth nonconvex penalty $\log(1+(\cdot)^2)$.
For simplicity, we use $\sigma = 1$ in the result below. 
\begin{claim}
\label{claim:tiber}
\begin{equation}
\label{eq:TiberChar}
\rho_{\mathrm{T}}(r;[\kappa,1]) = \min_s \log(1 + (s-r)^2) + \frac{2\kappa}{1+\kappa^2}|s|.
\end{equation}
\end{claim}
\begin{proof}
Denote the objective function in~\eqref{eq:TiberChar} by $f(s)$.
It is easy to check that $f$ is quasi-convex in $s$ when $\kappa \ge 0$.
We look to local optimality conditions to understand the structure of the minimizers.  
\begin{itemize}
\item Suppose $s^* > 0$. Then  $0 = f'(s^*)$ means
\begin{align*}
0=\frac{2(s^*-r)}{1+(s^*-r)^2} + \frac{2\kappa}{1+\kappa^2} & \iff s^* = r - \kappa;
\end{align*}
this requires $r > \kappa$.
\item Suppose $s^*<0$. Then $0 = f'(s^*)$ means
\begin{align*}
0=\frac{2(s^*-r)}{1+(s^*-r)^2} + \frac{-2\kappa}{1+\kappa^2} & \iff s^* = r + \kappa;
\end{align*}
this requires $r < -\kappa$.
\item otherwise $s^* = 0$.
\end{itemize}
Therefore $s^* = \mathbb{S}_\kappa(r)$. Plugging this into~\eqref{eq:TiberChar}, we have
\begin{align*}
\rho_{\mathrm{T}}(r;[\kappa,1]) &=  \log(1 + (\mathbb{S}_\kappa(r)-r)^2) + \frac{2\kappa}{1+\kappa^2}| \mathbb{S}_\kappa(r)|.
\end{align*}
\end{proof}
In Figure~\ref{fig:penalties}, we see that Tiber rises steeply near the origin. 
This behavior discourages dynamic terms (leaves, waves) in $R$, forcing them to be fit 
by $U\T V$.
The new Tiber-robust rPCA problem is given by:
\begin{equation}
\label{eq:rPCA_tiber}
\min_{U,V} \rho_{\mathrm{T}}(U\T V - Y;[\kappa,\sigma])
\end{equation}
which also has all of the advantages of \eqref{eq:rPCA_huber_ncvx}.
Moreover, because of the characterization from Claim~\ref{claim:tiber}, once we solve~\eqref{eq:rPCA_tiber},
we immediately recover $L$ and $S$:
\begin{align*}
L &= U\T V, \quad S = \mathbb{S}_{\kappa\sigma}(U\T V - Y).
\end{align*}

\subsection{Experiment: Foreground Separation}
We use a publicly available data set\footnote{Downloaded from \url{http://vis-www.cs.umass.edu/~narayana/castanza/I2Rdataset/}} with a dynamic background (moving trees). We sample 102 data frames from this data set, convert them to grey scale, and reshape them as column vectors of matrix $Y\in\R^{20480\times102}$.
We compare formulations \eqref{eq:rPCA_huber_ncvx} and \eqref{eq:rPCA_tiber}.
Proximal alternating linearized minimization algorithm (PALM) \cite{bolte2014proximal} was used to solve all of the optimization problems.

Rank of $U$ and $V$ was set to be $k=10$ for all experiments.
We manually tuned parameters to achieve the best possible recovery in each formulation. 
For Huber, we selected two nearby $\kappa$ values, $\kappa = 0.15$ and $\kappa = 0.1$; for Tiber, we selected $\kappa = 10$ and $\sigma = 0.03$, resulting in the threshold parameter $\kappa\sigma = 0.3$. 

The results are shown in Figure~\ref{fig:rpca}. The task is identifying the van while avoiding interference from moving leaves.
The Huber is unable to separate the van from the leaves for any threshold values $\kappa$.
When we choose $\kappa = 0.15$ (left panel in Figure~\ref{fig:rpca}), we cut out too much information, giving an
incomplete van in $S$. If we make a less conservative choice $\kappa = 0.1$ (middle panel in Figure~\ref{fig:rpca}), 
we leave too much dynamic noise in $S$, which obscures the van. 

The Tiber Penalty obtains a cleaner picture of the moving vehicle (right panel in Figure~\ref{fig:rpca}). 
As expected, it forces more of the dynamic background to be fit by $U\T V$, leaving a fairly complete 
van in $S$ without too much contamination.
\section{Robust Representation Learning for Clustering}
\label{sec:LatentClustering}

Centroid-based clustering, e.g. k-Means, is a standard tool to partition and summarize datasets. 
Given the high dimensionality and complexity of data in computer vision applications, 
it is necessary to learn latent representations,
such as the underlying metric, prior to clustering.
Clustering is then performed in the latent space. 

We develop an approach for robust spectral clustering.
We illustrate the advantages using a synthetic dataset, and then combine the approach with robust subspace clustering 
to achieve perfect performance on face recognition tasks. 
\subsection{Spectral Clustering}
\label{sec:spectral}
\begin{figure}[t]
\begin{center}
\includegraphics[width=0.7\linewidth]{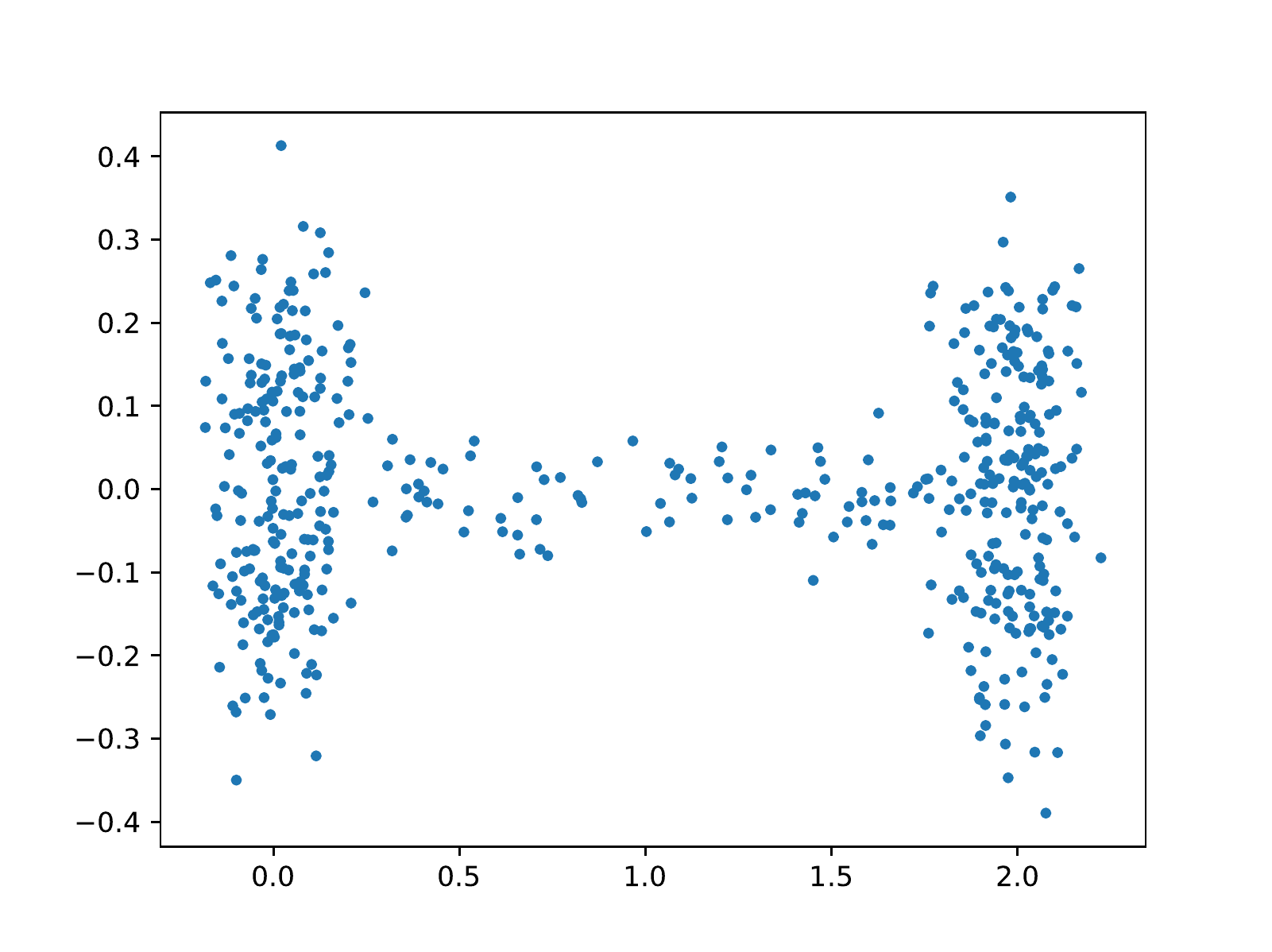}
\includegraphics[width=0.7\linewidth]{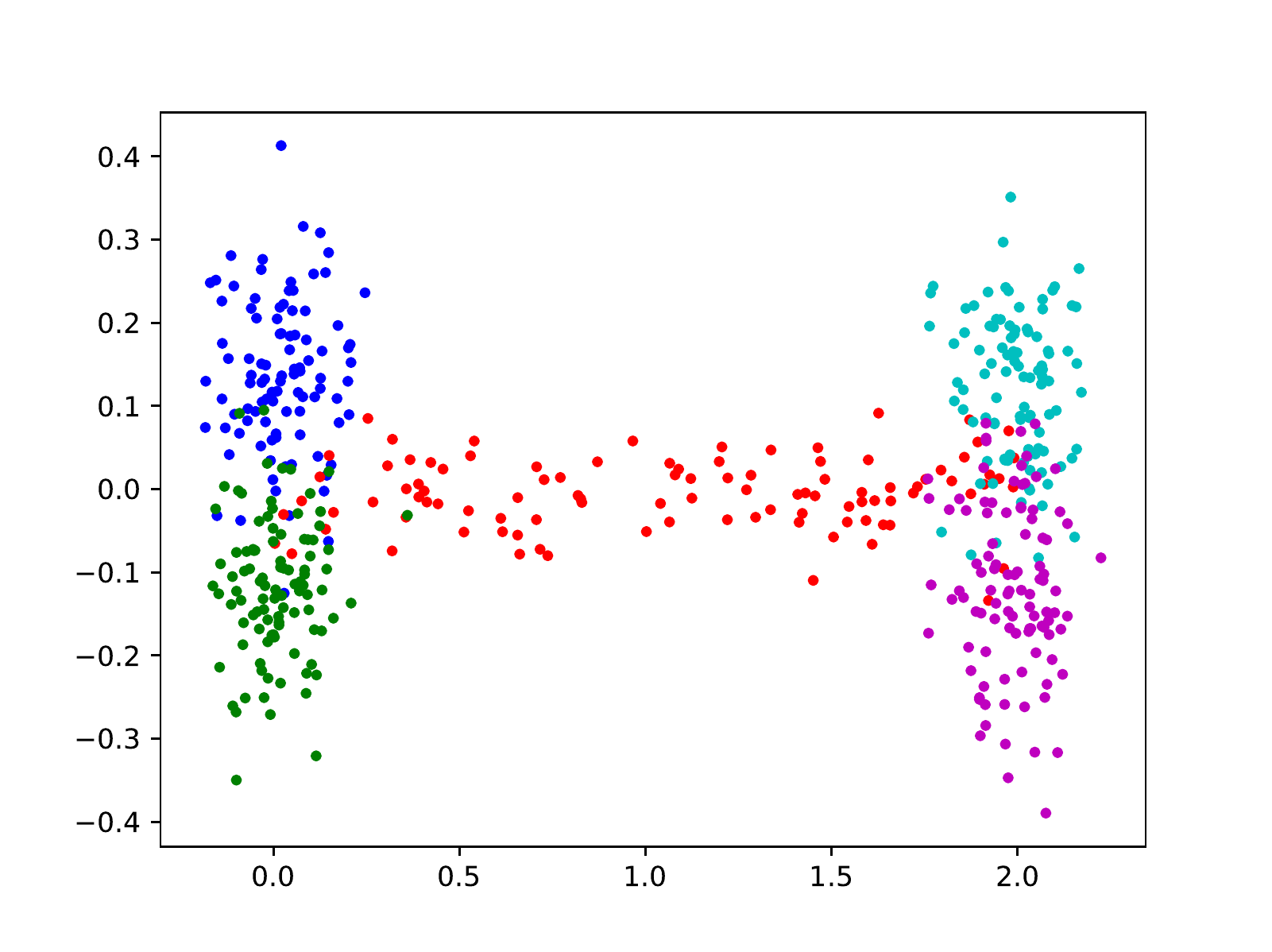}
\end{center}
   \caption{Synthetic Data Clustering: Up: data without labels, Down: data with true colors.}
\label{fig:synthetic}
\end{figure}
\begin{figure}[h]
\begin{center}
\includegraphics[width=0.7\linewidth]{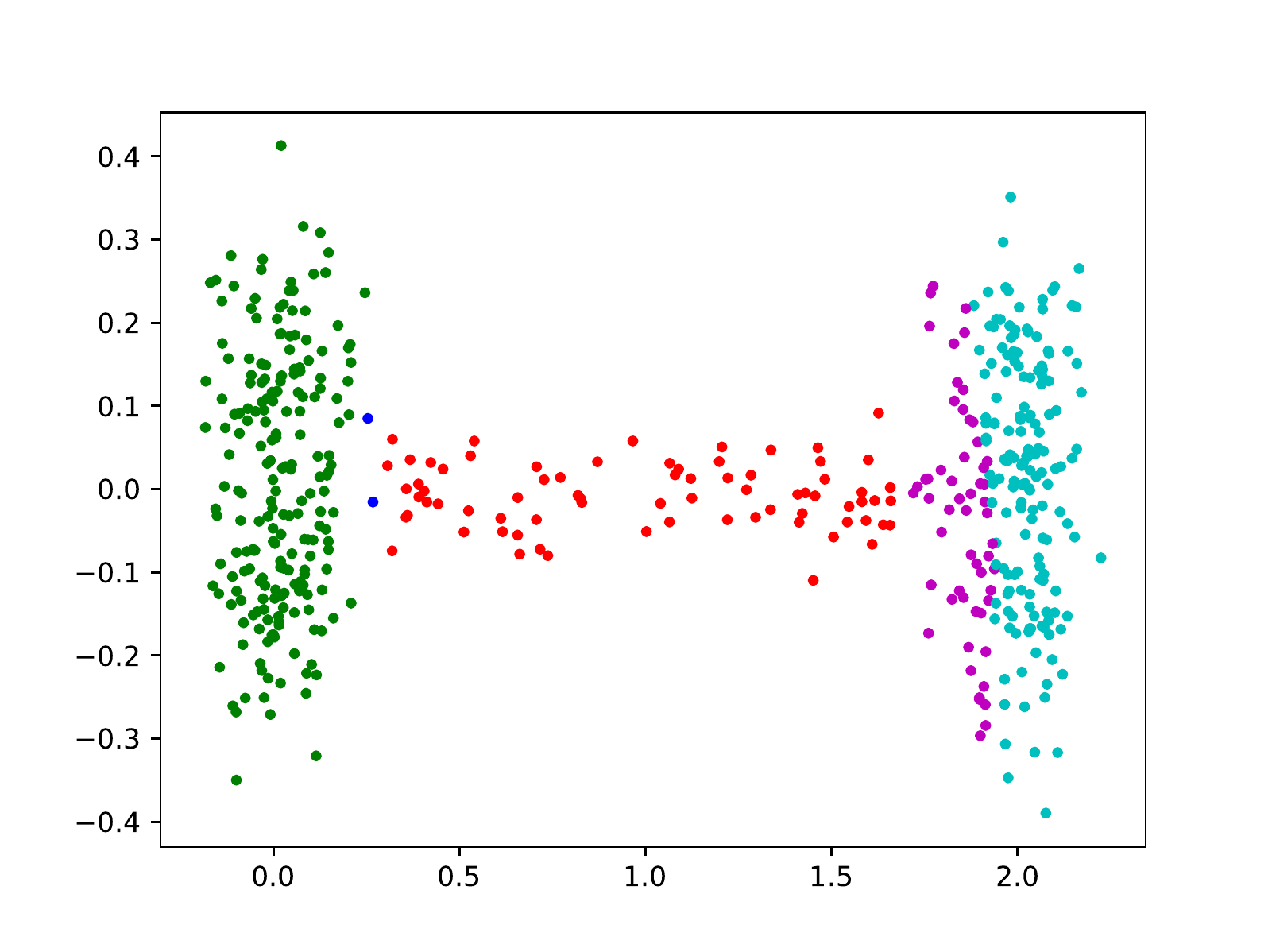}
\includegraphics[width=0.7\linewidth]{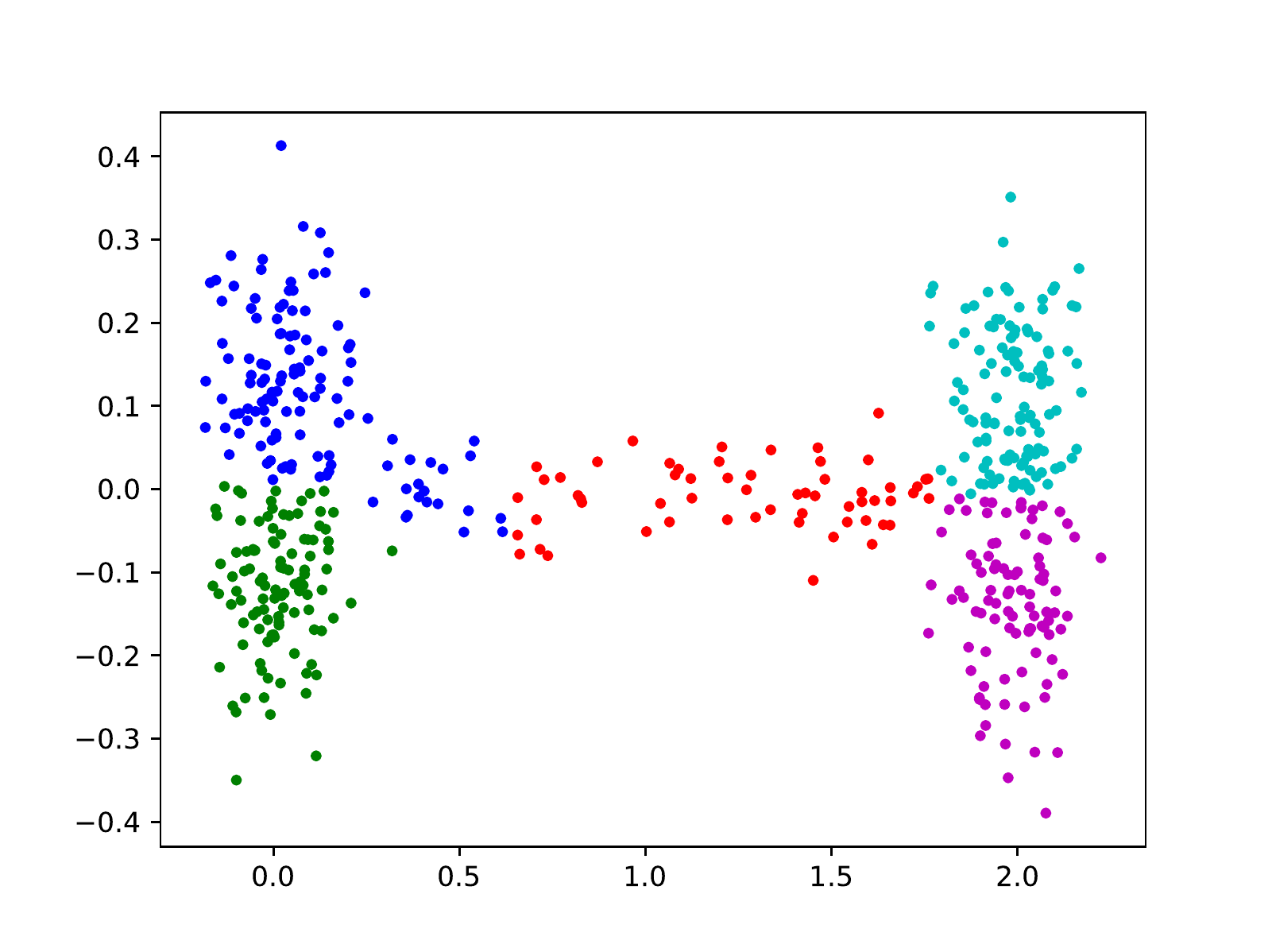}
\end{center}
   \caption{Synthetic Data Clustering: Up: result from eigenvalue decomposition, Down: result from \eqref{eq:symmf_huber}.}
\label{fig:synthetic_result}
\end{figure}

Spectral clustering \cite{ng2002spectral} is formulated as follows. Given $m$ datapoints $y_i \in \mathbb{R}^n$,
we arrange them in a matrix $Y\in\R^{n \times m}$.
To partition the data into $k$ groups, spectral clustering uses the following steps: 
\begin{enumerate}
\item Given a dataset of $m$ samples, we construct the similarity matrix $L\in\R^{m\times m}$ of the data points.
\item Extract the eigenvectors $X\in\R^{m \times k}$ of $L$ corresponding to the $k$ largest eigenvalues.
\item Project each row of $X$ onto the unit ball, and apply distance-based clustering.
\end{enumerate}
Finding a meaningful similarity matrix $L$ is crucial to the success of spectral clustering. Ideally, 
$L$ will be a block diagonal matrix with $k$ blocks. This rarely happens for real applications; 
even when underlying structure in $L$ is present, it can be obscured by noise and a small number of 
points that don't follow the general pattern.

To find a factorization of noisy $L$, we need a robust method for eigenvalue decomposition. 
We first formulate eigenvalue decomposition as an optimization problem.
\begin{claim}
Assume $L$ is a symmetric matrix with eigenvalues less than or equal to 1.
Then the solution to the problem 
\begin{equation}
\label{eq:symmf}
\begin{aligned}
\min_X~&\frac{1}{2}\|XX\T - L\|_F^2\\
\text{s.t.}~&X\T X = I_{k}
\end{aligned}
\end{equation}
is $X = [v_1, \dots, v_k]$ with $v_i$ the eigenvector corresponding to the 
 $i^\text{th}$ largest eigenvalue of $L$, and $I_k$ the $k$ by $k$ identity matrix.
\end{claim}
\begin{proof}
Since $L$ is a symmetric matrix, it has a eigenvalue decomposition,
\[L = Y\Lambda Y\T\]
where $Y\in\R^{m\times m}$ is orthogonal and $\Lambda$ is diagonal,
with $1\ge\lambda_1\ge\lambda_2\ge\ldots\ge\lambda_m$. Similarly, we have 
\[XX\T = \tilde{X}D\tilde{X}\T\]
where $\tilde{X}\in\R^{m\times m}$ is a orthogonal matrix whose first $k$ columns agree with those of $X$,
$D$ is a diagonal matrix with first $k$ elements on the diagonal are 1 and the rest are 0.
From the Cauchy-€"Schwarz inequality, we have
\[\trace(XX\T\cdot L) = \ip{XX\T,L} \le \|XX\T\|_F\cdot\|L\|_F\]
where equality hold when $XX\T$ and $L$ share the same singular vectors, i.e.,
$X$ equals to the first $k$ columns of $Y$.\\
Therefore
\begin{align*}
\frac{1}{2}\|XX\T - L\|_F^2 &= \frac{1}{2}\|XX\T\|_F^2 - \ip{XX\T,L} + \frac{1}{2}\|L\|_F^2\\
&\ge\frac{1}{2}\|D\|_F^2 - \|D\|_F\|\Lambda\|_F + \frac{1}{2}\|\Lambda\|_F^2
\end{align*}
with equality hold when columns of $X$ are eigenvectors corresponding to the largest $k$ eigenvalues.
\end{proof}

We robustify \eqref{eq:symmf} by replacing the Frobenius norm in the optimization formulation by  
the Huber function (or another robust penalty):
\begin{equation}
\label{eq:symmf_huber}
\begin{aligned}
\min_X~&\rho(XX\T - L;\kappa)\\
\text{s.t.}~&X\T X = I_{k}
\end{aligned}.
\end{equation}
This approach can be very effective. Consider the following clustering experiment with $n=2$, $m = 500$, and 
$k=5$. 
We generate five clusters (sampling from four 2-D Gaussians, one rectangular uniform distribution) 
with 100 points per group. To make the problem challenging, we move the clusters close together so much that trying to tell them apart with the naked eye is hard~(Figure~\ref{fig:synthetic}, top). 
True clusters appear in Figure~\ref{fig:synthetic}, bottom. 

Classic spectral clustering, which uses eigenvalue decomposition in step 2, fails to detect the true relationships~(Figure~\ref{fig:synthetic_result}, top).
Robust spectral clustering using the Huber penalty~\eqref{eq:symmf_huber} 
does a much better job (Figure~\ref{fig:synthetic_result}, bottom).

\subsection{Subspace Clustering}
Subspace clustering looks for low dimensional representation of high dimensional data,
by grouping the points along low-dimensional subspaces. 
Given a data matrix $Y\in\R^{n\times m}$ as in Section~\ref{sec:spectral},
the optimization for subspace clustering is given by~\cite{elhamifar2013sparse}:
\begin{equation}
\label{eq:subspace}
\begin{aligned}
\min_C~&\frac{1}{2}\|Y - YC\|_F^2 + \lambda\|C\|_1 \;\; \text{s.t.} \;\; \diag(C) = 0.
\end{aligned}
\end{equation}
This formulation looks for a sparse representation of the dataset by its members:
\(s_i = Sc_i\).
To avoid the trivial solution, we require the diagonal of $C$ to be identically 0. After obtaining $C$, it is post-processed and a similarity matrix is constructed as $W = |C| + |C\T|$. $W$ will be ideally close to block-diagonal, where each block represents a subspace, and spectral clustering is performed it to identify cluster memberships.

Outliers in the dataset can break the performance of \eqref{eq:subspace}.
To make the approach robust,~\cite{elhamifar2013sparse} uses the formulation 
\begin{equation}
\label{eq:subspace_robust}
\begin{aligned}
\min_C~&\frac{1}{2}\|Y - YC - S\|_F^2 + \lambda\|C\|_1 + \kappa\|S\|_1\\
\text{s.t.}~&\diag(C) = 0.
\end{aligned}
\end{equation}
Using Claim~\ref{claim:huber},
we rewrite \eqref{eq:subspace_robust} using Huber:
\begin{equation}
\label{eq:subspace_huber}
\begin{aligned}
\min_C~&\rho(Y - YC;\kappa) + \lambda\|C\|_1 \;\; \text{s.t.}\;\; \diag(C) = 0.
\end{aligned}
\end{equation}
Formulation~\eqref{eq:subspace_huber} has the same advantages with respect to~\eqref{eq:subspace_robust} as \eqref{eq:rPCA_huber_ncvx} 
has with respect to \eqref{eq:rPCA}.

\subsection{Face Clustering}
Given multiple face images taken at different conditions, the goal of face clustering~\cite{elhamifar2013sparse} is to identify images that belong to the same person.

\begin{figure}[h]
\begin{center}
\includegraphics[width=1\linewidth]{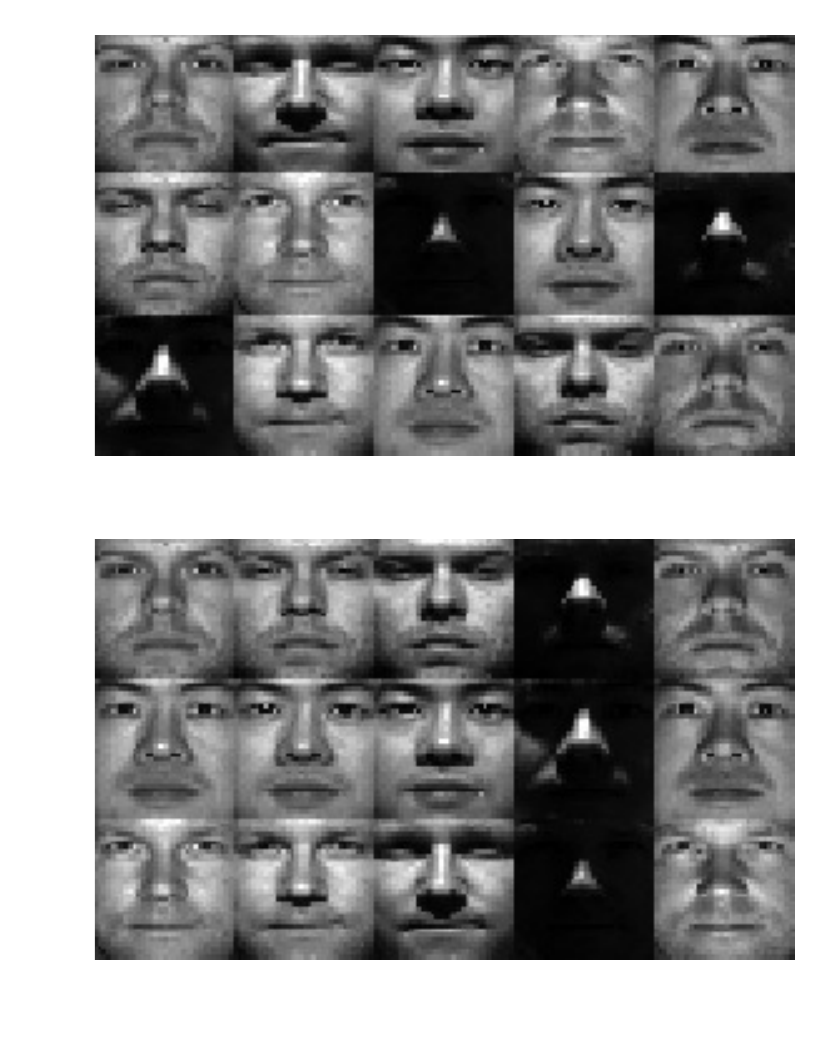}
\end{center}
   \caption{Faces data: top: randomly chosen face images, bottom: faces after clustering; each row belongs to a cluster.}
\label{fig:faces}
\end{figure}

We use images from the publicly available Extended Yale B dataset \cite{lee2005acquiring} 
\footnote{Downloaded from \url{http://www.cad.zju.edu.cn/home/dengcai/Data/FaceData.html}}.
Each image has $32\times 32$ pixels, and there are 2414 images in the dataset.
These images belong to 38 people, with approximately 64 pictures per person. 

Under the Lambertian assumption, pictures obtained from one person under different illuminations 
should lie close to a 9 dimensional subspace~\cite{basri2003lambertian}.
In practice, these spaces are hard to detect because of noise in the images, and 
a robust approach is required.

\noindent{\bf Robust subspace clustering for face images}:
\vspace{-.1in}
\begin{enumerate}
\item Obtain sparse representation $C$ using \eqref{eq:subspace_huber}.
\item Construct similarity matrix $W$ from $C$.
\vspace{-.1in}
\begin{itemize}
\item Normalize columns of $C$ to have maximum absolute value no larger than 1. 
\item Form $W = |C| + |C\T|$
\item Normalize $W$: $W \leftarrow D^{-1/2}WD^{-1/2}$, where $D$ is a diagonal matrix with $D_{ii} = \sum_j W_{ij}$.
\end{itemize}

\item Apply spectral clustering using $W$.
\vspace{-.1in}
\begin{itemize}
\item Apply robust symmetric factorization \eqref{eq:symmf_huber} to $W$, to obtain the latent representation $X$.
\item Project each row of $X$ onto the unit 2-norm ball.
\item Apply K-means algorithm to the new rows of $X$.
\end{itemize}
\end{enumerate}

The results are shown in Table~\ref{tb:faces}. We implement the approach for different numbers of subjects $k = 2,3,5,8$.
We show the parameters $\kappa$ and $\lambda$ in \eqref{eq:subspace_huber} used to achieve 
the high accuracies given in Table~\ref{tb:faces}\footnote{In \cite{elhamifar2013sparse}, the images used are of size $48 \times 42$. The numbers shown are therefore indicative.}.
\begin{table}[h]
\begin{center}
\begin{tabular}{|c|c|c|c|c|}
\hline
clusters & $\kappa$ in \eqref{eq:subspace_huber} & $\lambda$ in 
\eqref{eq:subspace_huber} & error & error in \cite{elhamifar2013sparse}\\
\hline\hline
$k=2$ & 0.5 & 1 & 0.00\%& 1.86\%\\
$k=3$ & 0.1 & 0.7 & 0.00\%& 3.10\%\\
$k = 5$ & 0.05 & 0.7 & 0.00\% & 4.31\%\\
$k = 8$ & 0.03 & 0.5 & 2.73\% & 5.85\%\\
\hline
\end{tabular}
\end{center}
\caption{\label{tb:faces} Results for robust subspace clustering with face images.}
\end{table}

To get better intuition of the method, we plot the similarity matrix 
corresponding to $k = 3$ in Figure~\ref{fig:similarity}.
We can clearly see three blocks along the diagonal that correspond to the three face clusters.
\begin{figure}[h]
\begin{center}
\includegraphics[width=0.8\linewidth]{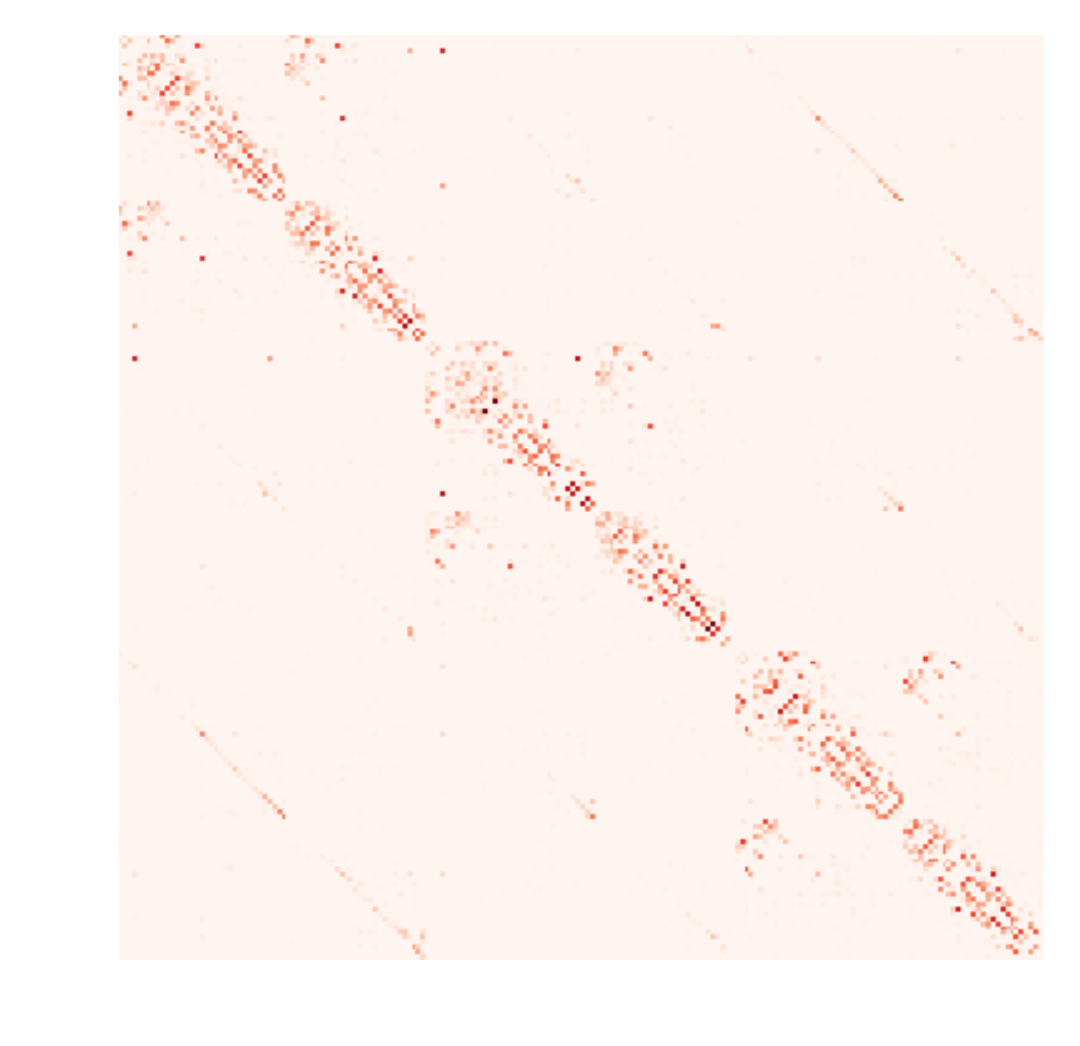}
\end{center}
   \caption{Similarity matrix for face images clustering with $k=3$; the matrix is nearly block diagonal with 3 blocks.}
\label{fig:similarity}
\end{figure}
The resulting projected $X$ obtained from the eigenvalue decomposition of similarity matrix $W$ are shown in Figure~\ref{fig:spectrum}. The three clusters are clearly well separated. 
\begin{figure}[h]
\begin{center}
\includegraphics[width=0.9\linewidth]{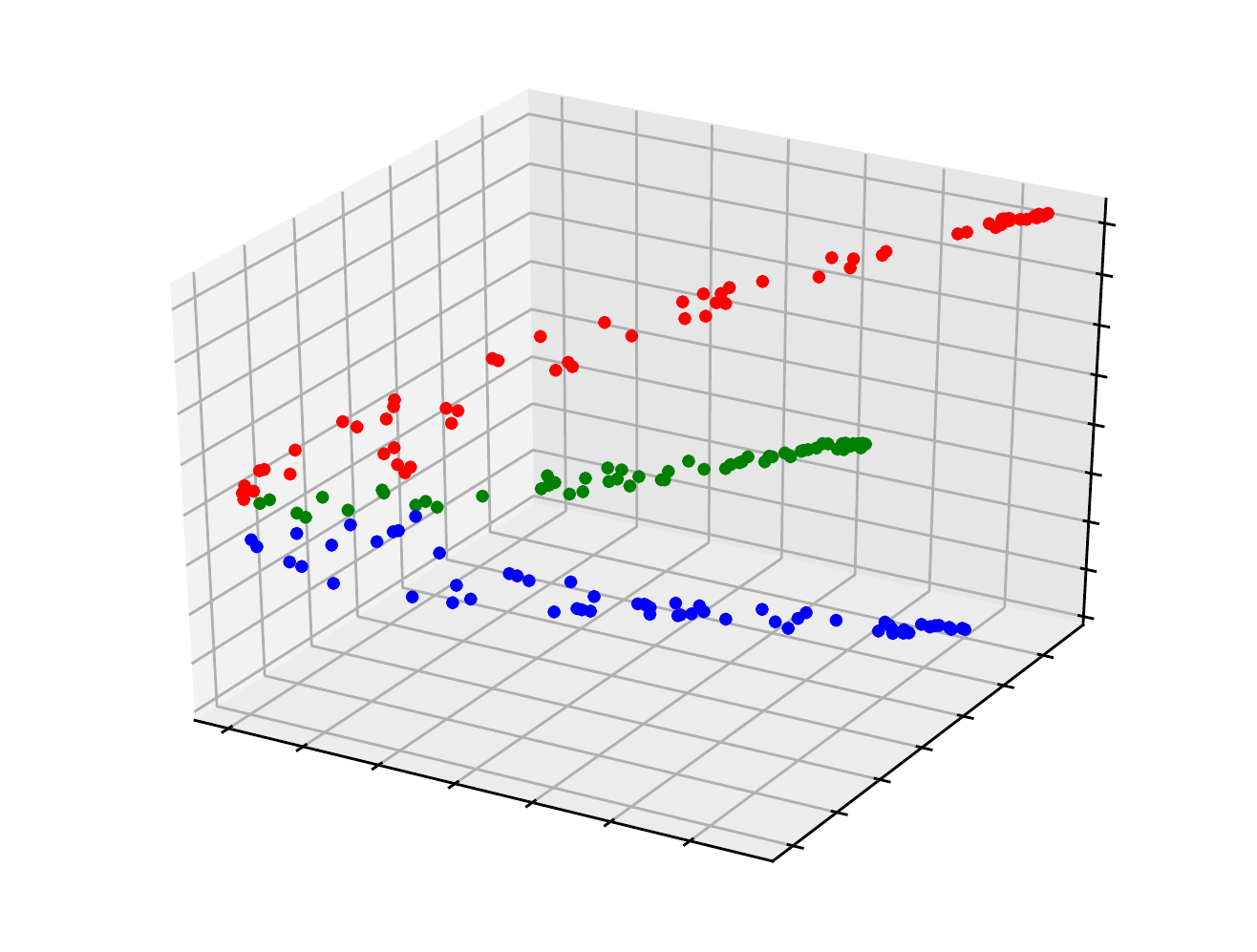}
\end{center}
   \caption{Projections of the rows of $X$ onto the eigenspace of the similarity matrix for $k=3$.
   Each color represent the face images of a single person.}
\label{fig:spectrum}
\end{figure}
The final algorithm has perfect accuracy in this example.
\section{Discussion}

 Robust approaches are essential for unsupervised learning, and can be designed  
 using optimization formulations. For example, in both 
 rPCA and robust spectral learning, SVD and eigenvalue decomposition are 
 first characterized using optimization, then reformulated with robust losses.  
 
 Several tasks in this approach are difficult. First, there is a need to tune parameters
 in the optimization formulations. For example, the Tiber depends on two parameters,
 $\kappa$ and $\sigma$. Automatic ways to tune these parameters can make 
 robust unsupervised learning a lot more portable. Second, the optimization problems 
we have to solve are large-scale; time required for robust subspace clustering for images scales 
non-linearly with both the number and size of images. Designing non-smooth stochastic algorithms 
that take the structure of these problems into account is essential. 
{\small
\bibliographystyle{IEEEtran}
\bibliography{egbib}
}

\end{document}